\documentclass{article}
\usepackage{arxiv}

\usepackage{authblk}

\usepackage[american]{babel}

\usepackage{natbib} 
\usepackage{mathtools} 
\usepackage{booktabs} 
\usepackage{tikz} 
\usepackage{pgfplots}
\usetikzlibrary{pgfplots.groupplots}

\usepackage{amsmath,amsfonts,amssymb,dsfont,bm,amsthm}

\usepackage{hyperref}

\usepackage{csquotes}
\usepackage{mathbbol}
\usepackage{subcaption}
\usepackage{thmtools,thm-restate}
\usepackage{MnSymbol}
\usepackage{bm}
\usepackage{accents}

\usepackage{mathtools} 
\usepackage{booktabs} 
\usepackage{tikz} 
\usepackage{amsfonts,dsfont,bm}

\usepackage{enumitem}

\usepackage{subcaption}

\newtheorem{definition}{Definition}

\newtheorem{assumption}{Assumption}
\newtheorem{lemma}{Lemma}[section]
\newtheorem{theorem}{Theorem}

\newtheorem{problem}{Problem}

\newcommand{\comment}[1]{}

\newcommand{\ie}{\textit{i.e.}}

\newcommand{\interv}[1]{\text{do}(#1)}
\newcommand{\probasymbol}{\Pr}
\newcommand{\proba}[1]{\probasymbol(#1)}
\newcommand{\probac}[2]{\proba{#1 \mid #2}}

\colorlet{green}{black}

\title{Causal reasoning in difference  graphs}

\author[1]{Charles K. Assaad}

\affil[1]{Sorbonne Université, INSERM, Institut Pierre Louis d’Epidémiologie et de Santé Publique, F75012, Paris, France}

\date{}

\begin{document}
\maketitle

\begin{abstract}
Understanding causal mechanisms across different populations is essential for designing effective public health interventions. Recently, difference graphs have been introduced as a tool to visually represent causal variations between two distinct populations. While there has been progress in inferring these graphs from data through causal discovery methods, there remains a gap in systematically leveraging their potential to enhance causal reasoning. This paper addresses that gap by establishing conditions for identifying causal changes and effects using difference graphs. It specifically focuses on identifying total causal changes and total effects in a nonparametric setting, as well as direct causal changes and direct effects in a linear setting. In doing so, it provides a novel approach to  causal reasoning that holds potential for various public health applications.
\end{abstract}

\section{Introduction}
In epidemiology, it is critical to understand the causal mechanisms that impact health outcomes across diverse populations, which may differ based on geographic location, socioeconomic status, genetic predispositions, or environmental exposures.
Recently, difference graphs have been introduced to represent differences between two Structural Causal Models (SCMs) \citep{Pearl_2000}, with each model corresponding to a distinct population.
These graphs offer a refined representation to visually compare and analyze the causal variations between these groups.
By clearly delineating these causal differences, epidemiologists can tailor public health policies and interventions more effectively to address the unique needs of each population.
Moreover, it has been highlighted by \cite{Malik_2024} that such graphs are also interesting for root cause analysis~\citep{Assaad_2023}. 

Recently, there has been a surge of interest in constructing difference graphs directly from data  through methods rooted in causal discovery~\citep{Wang_2018,Chen_2023,Malik_2024, Bystrova_2024_b}. 
The simplest approach involves independently discovering the SCM for each of the two populations and then computing the difference graph. However, a more effective strategy has emerged that involves directly constructing the difference graph, bypassing the two SCMs that define it. This approach reduces the need for stringent assumptions commonly associated with causal discovery~\citep{Spirtes_2000,Glymour_2019}.
Despite these advancements, an evident gap remains in the existing literature: there has been no systematic investigation into how difference graphs themselves can facilitate or enhance the process of causal reasoning. This oversight presents a unique opportunity to explore the potential of difference graphs as a tool not just for representation but also for driving identifiability of causal changes and causal effects.

To bridge this gap, this work establishes both sufficient and necessary conditions for identifying total causal changes and total effects in difference graphs within nonparametric models, under the assumption of no hidden confounding. 
Additionally, this work provides sufficient and necessary conditions for identifying direct causal changes and direct effects in difference graphs within linear models, under the assumption of no hidden confounding.
Initially, both results are presented under the assumption that the two causal graphs, induced by the two SCMs defining the difference graph, share the same topological ordering. These results are then extended to accommodate cases where the graphs have different topological orderings.

In the remainder, Section~\ref{sec:DiffG} introduces difference graphs along with necessary tools and terminologies. Section~\ref{sec:identification_main_total} details the identifiability results for total effects. Section~\ref{sec:identification_main_direct} focuses on the identifiability of direct effects. 
Section~\ref{sec:rw} presents related works and offers further discussion. Finally, Section~\ref{sec:conclusion} concludes the paper.
In the following, uppercase letters represent variables and lowercase letters represent their values, while letters in blackboard bold denote sets.

\section{Difference graphs, causal changes, and causal effects}
\label{sec:DiffG}

This study utilizes the framework of structural causal models (SCMs) introduced by \cite{Pearl_2000}. Formally, an SCM $\mathcal{M}$ is defined as a 4-tuple $( \mathbb{U}, \mathbb{V}, \mathbb{F}, P(\mathbb{U}) )$, where $\mathbb{U}$ represents a set of exogenous (latent) variables and $\mathbb{V}$ denotes a set of endogenous (observed) variables. The set $\mathbb{F}$ consists of functions $\{ f_i \}_{i=1}^{|\mathbb{V}|}$, where each endogenous variable is determined by a function $f_i \in \mathbb{F}$ that depends on its corresponding exogenous variable $U_i\in \mathbb{U}$ and a subset of endogenous variables (which are also called the direct causes or the parents). In a linear SCM, $f_i$ is a linear equation where each coefficient is called path coefficient~\citep{Wright_1920,Wright_1921} and it designates a direct effect. The following assumption regarding exogenous variables is made throughout the paper.
\begin{assumption}
	\label{assumption:hidden}
	All exogenous variables are mutually independent, i.e., no hidden confounding.
\end{assumption}
Consistent with common practice in causal inference, it is also assumed that each SCM $\mathcal{M}$ induces a causal graph $\mathcal{G}=(\mathbb{V}, \mathbb{E})$  which is assumed to  be a directed acyclic graph (DAG). For clarity, any causal graph that is a DAG will be referred to as a causal DAG. In this graph, $\mathbb{V}$ represents the set of vertices, each corresponding to a random variable, and $\mathbb{E}$ includes the directed edges, with each edge $X\rightarrow Y$ indicating a causal relationship between $X$ and $Y$. 
For any graph $\mathcal{G} = (\mathbb{V}, \mathbb{E})$, standard terminology is adopted to describe relations to any vertex  $Y \in \mathbb{V}$: 
$Pa(Y, \mathcal{G})$ denotes the set of vertices from which there is a direct edge to $Y$; $An(Y, \mathcal{G})$ refers to the ancestors of $Y$ which is the collective set including $Y$ and the parents of all vertices that can be recursively traced back from $Y$; $De(Y, \mathcal{G})$ refers to the descendants of $Y$ which comprises all vertices $X$ for which $Y$ is found within $An(X, \mathcal{G})$. If a path contains the subpath $X \rightarrow W \leftarrow Y$, then $W$ is a \emph{collider} on the path. A path from $X$ to $Y$ is considered \emph{active} given a vertex set $\mathbb{W}$, where neither $X$ nor $Y$ are in $\mathbb{W}$, if every non-collider on the path is not in $\mathbb{W}$, and every collider on the path has a descendant in $\mathbb{W}$. Otherwise, the path is \emph{blocked} by $\mathbb{W}$.
It is important to note that many SCMs can induce the same causal DAG and that a  causal DAG can be compatible with many probability distributions. 

\begin{figure*}[t!]
	\centering
	\begin{minipage}{.3\textwidth}
		\begin{subfigure}{\textwidth}
	\centering
	\caption{Causal  DAG 1.}
			\begin{tikzpicture}[{black, circle, draw, inner sep=0}]
				\tikzset{nodes={draw,rounded corners},minimum height=0.5cm,minimum width=0.5cm, font=\footnotesize}	
				\tikzset{latent/.append style={white, fill=black}}
				
				\node[fill=red!30] (X) at (0,0) {$X$} ;
				\node[fill=blue!30] (Y) at (3.5,0) {$Y$};
				
				\draw [->,>=latex] (X) -- (Y);
				
			\end{tikzpicture}
			
		\end{subfigure}
		
		\begin{subfigure}{\textwidth}
	\centering
	\caption{Causal  DAG 2.}
			\begin{tikzpicture}[{black, circle, draw, inner sep=0}]
				\tikzset{nodes={draw,rounded corners},minimum height=0.5cm,minimum width=0.5cm, font=\footnotesize}	
				\tikzset{latent/.append style={white, fill=black}}
				
				\node[fill=red!30] (X) at (0,0) {$X$} ;
				\node[fill=blue!30] (Y) at (3.5,0) {$Y$};
				
				\draw [->,>=latex] (X) -- (Y);		
			\end{tikzpicture}
			
		\end{subfigure}
	\end{minipage}
	\hfill
	\vline
	\hfill
	\begin{minipage}{.3\textwidth}
		\begin{subfigure}{\textwidth}
	\centering
	\caption{Difference graph.}
			\begin{tikzpicture}[{black, circle, draw, inner sep=0}]
				\tikzset{nodes={draw,rounded corners},minimum height=0.5cm,minimum width=0.5cm, font=\footnotesize}	
				\tikzset{latent/.append style={white, fill=black}}
				
				\node[fill=red!30] (X) at (0,0) {$X$} ;
				\node[fill=blue!30] (Y) at (3.5,0) {$Y$};
				
			\end{tikzpicture}
			
		\end{subfigure}
	\end{minipage}
	\hfill
	\vline
	\hfill
	\begin{minipage}{.3\textwidth}
		\begin{subfigure}{\textwidth}
	\centering
	\caption{Other causal  DAG 1.}
			\begin{tikzpicture}[{black, circle, draw, inner sep=0}]
				\tikzset{nodes={draw,rounded corners},minimum height=0.5cm,minimum width=0.5cm, font=\footnotesize}	
				\tikzset{latent/.append style={white, fill=black}}
				
				\node[fill=red!30] (X) at (0,0) {$X$} ;
				\node[fill=blue!30] (Y) at (3.5,0) {$Y$};

				\draw [->,>=latex] (Y) -- (X);
				
			\end{tikzpicture}
			
		\end{subfigure}
		
		\begin{subfigure}{\textwidth}
	\centering
	\caption{Other causal  DAG 2.}
			\begin{tikzpicture}[{black, circle, draw, inner sep=0}]
				\tikzset{nodes={draw,rounded corners},minimum height=0.5cm,minimum width=0.5cm, font=\footnotesize}	
				\tikzset{latent/.append style={white, fill=black}}
				
				\node[fill=red!30] (X) at (0,0) {$X$} ;
				\node[fill=blue!30] (Y) at (3.5,0) {$Y$};
				
				\draw [->,>=latex] (Y) -- (X);
			\end{tikzpicture}
			
		\end{subfigure}
	\end{minipage}
	
	\hrulefill
	%
	%
	%
	%
	%
	%

	\begin{minipage}{.3\textwidth}
		\begin{subfigure}{\textwidth}
				\centering
				\caption{Causal  DAG 1.}
			\centering
			\begin{tikzpicture}[{black, circle, draw, inner sep=0}]
				\tikzset{nodes={draw,rounded corners},minimum height=0.5cm,minimum width=0.5cm, font=\footnotesize}	
				\tikzset{latent/.append style={white, fill=black}}
				
				\node[fill=red!30] (X) at (0,0) {$X$} ;
				\node[fill=blue!30] (Y) at (3.5,0) {$Y$};
				\node (Z) at (1,0.5) {$W_1$};
				\node (W) at (2.5, 0.5) {$W_2$};

				\draw [<-,>=latex] (X) -- (Z);
				\draw [->,>=latex] (X) -- (W);
				
				\draw [->,>=latex] (Z) -- (Y);
				\draw [->,>=latex] (W) -- (Y);
				
				\draw [->,>=latex] (X) -- (Y);
				
			\end{tikzpicture}
			
		\end{subfigure}
		
		\begin{subfigure}{\textwidth}
	\centering
	\caption{Causal  DAG 2.}
			\begin{tikzpicture}[{black, circle, draw, inner sep=0}]
				\tikzset{nodes={draw,rounded corners},minimum height=0.5cm,minimum width=0.5cm, font=\footnotesize}	
				\tikzset{latent/.append style={white, fill=black}}
				
				\node[fill=red!30] (X) at (0,0) {$X$} ;
				\node[fill=blue!30] (Y) at (3.5,0) {$Y$};
				\node (Z) at (1,0.5) {$W_1$};
				\node (W) at (2.5, 0.5) {$W_2$};

				\draw [->,>=latex] (Z) -- (Y);
				\draw [->,>=latex] (W) -- (Y);
				
				
			\end{tikzpicture}
			
		\end{subfigure}
	\end{minipage}
	\hfill
	\vline
	\hfill
	\begin{minipage}{.3\textwidth}
		\begin{subfigure}{\textwidth}
				\centering
				\caption{Difference graph.}
			\begin{tikzpicture}[{black, circle, draw, inner sep=0}]
				\tikzset{nodes={draw,rounded corners},minimum height=0.5cm,minimum width=0.5cm, font=\footnotesize}	
				\tikzset{latent/.append style={white, fill=black}}
				
				\node[fill=red!30] (X) at (0,0) {$X$} ;
				\node[fill=blue!30] (Y) at (3.5,0) {$Y$};
				\node (Z) at (1,0.5) {$W_1$};
				\node (W) at (2.5, 0.5) {$W_2$};

				\draw [<-,>=latex] (X) -- (Z);
				\draw [->,>=latex] (X) -- (W);
				
				
				\draw [->,>=latex] (X) -- (Y);
				
			\end{tikzpicture}
			
		\end{subfigure}
	\end{minipage}
	\hfill
	\vline
	\hfill
	\begin{minipage}{.3\textwidth}
					\begin{subfigure}{\textwidth}
				\centering
				\caption{Other causal  DAG 1.}
			\begin{tikzpicture}[{black, circle, draw, inner sep=0}]
				\tikzset{nodes={draw,rounded corners},minimum height=0.5cm,minimum width=0.5cm, font=\footnotesize}	
				\tikzset{latent/.append style={white, fill=black}}
				
				\node[fill=red!30] (X) at (0,0) {$X$} ;
				\node[fill=blue!30] (Y) at (3.5,0) {$Y$};
				\node (Z) at (1,0.5) {$W_1$};
				\node (W) at (2.5, 0.5) {$W_2$};

				\draw [<-,>=latex] (X) -- (Z);
				\draw [->,>=latex] (X) -- (W);
				
				\draw [->,>=latex] (Z) -- (Y);
				\draw [->,>=latex] (Y) -- (W);
				
				\draw [->,>=latex] (X) -- (Y);
				
			\end{tikzpicture}
			
		\end{subfigure}
		
		\begin{subfigure}{\textwidth}
	\centering
	\caption{Other causal  DAG 2.}
			\begin{tikzpicture}[{black, circle, draw, inner sep=0}]
				\tikzset{nodes={draw,rounded corners},minimum height=0.5cm,minimum width=0.5cm, font=\footnotesize}	
				\tikzset{latent/.append style={white, fill=black}}
				
				\node[fill=red!30] (X) at (0,0) {$X$} ;
				\node[fill=blue!30] (Y) at (3.5,0) {$Y$};
				\node (Z) at (1,0.5) {$W_1$};
				\node (W) at (2.5, 0.5) {$W_2$};
				
				\draw [->,>=latex] (Z) -- (Y);
				\draw [->,>=latex] (Y) -- (W);
				
				
			\end{tikzpicture}
			
		\end{subfigure}
	\end{minipage}
	
	%
	%
	
	\hrulefill
	
	\begin{minipage}{.3\textwidth}
		\begin{subfigure}{\textwidth}
	\centering
	\caption{Causal  DAG 1.}
			\begin{tikzpicture}[{black, circle, draw, inner sep=0}]
				\tikzset{nodes={draw,rounded corners},minimum height=0.5cm,minimum width=0.5cm, font=\footnotesize}	
				\tikzset{latent/.append style={white, fill=black}}
				
				\node[fill=red!30] (X) at (0,0) {$X$} ;
				\node[fill=blue!30] (Y) at (3.5,0) {$Y$};
				\node (Z) at (1,0.5) {$W_1$};
				\node (W) at (2.5, 0.5) {$W_2$};

				\draw [<-,>=latex] (X) -- (Z);
				\draw [->,>=latex] (X) -- (W);
				
				\draw [->,>=latex] (Z) -- (Y);
				\draw [->,>=latex] (W) -- (Y);
				
				\draw [->,>=latex] (X) -- (Y);
				
			\end{tikzpicture}
			
		\end{subfigure}
		
		\begin{subfigure}{\textwidth}
	\centering
	\caption{Causal  DAG 2.}
			\begin{tikzpicture}[{black, circle, draw, inner sep=0}]
				\tikzset{nodes={draw,rounded corners},minimum height=0.5cm,minimum width=0.5cm, font=\footnotesize}	
				\tikzset{latent/.append style={white, fill=black}}
				
				\node[fill=red!30] (X) at (0,0) {$X$} ;
				\node[fill=blue!30] (Y) at (3.5,0) {$Y$};
				\node (Z) at (1,0.5) {$W_1$};
				\node (W) at (2.5, 0.5) {$W_2$};

				\draw [->,>=latex] (Z) -- (Y);
				\draw [->,>=latex] (X) -- (W);
				
				
			\end{tikzpicture}
			
		\end{subfigure}
	\end{minipage}
	\hfill
	\vline
	\hfill
	\begin{minipage}{.3\textwidth}
		\begin{subfigure}{\textwidth}
	\centering
	\caption{Difference graph.}
			\begin{tikzpicture}[{black, circle, draw, inner sep=0}]
				\tikzset{nodes={draw,rounded corners},minimum height=0.5cm,minimum width=0.5cm, font=\footnotesize}	
				\tikzset{latent/.append style={white, fill=black}}
				
				\node[fill=red!30] (X) at (0,0) {$X$} ;
				\node[fill=blue!30] (Y) at (3.5,0) {$Y$};
				\node (Z) at (1,0.5) {$W_1$};
				\node (W) at (2.5, 0.5) {$W_2$};

				\draw [<-,>=latex] (X) -- (Z);
				\draw [->,>=latex] (W) -- (Y);
				
				
				\draw [->,>=latex] (X) -- (Y);
				
			\end{tikzpicture}
			
		\end{subfigure}
	\end{minipage}
	\hfill
	\vline
	\hfill
	\begin{minipage}{.3\textwidth}
		\begin{subfigure}{\textwidth}
	\centering
	\caption{Other causal  DAG 1.}
			\begin{tikzpicture}[{black, circle, draw, inner sep=0}]
				\tikzset{nodes={draw,rounded corners},minimum height=0.5cm,minimum width=0.5cm, font=\footnotesize}	
				\tikzset{latent/.append style={white, fill=black}}
				
				\node[fill=red!30] (X) at (0,0) {$X$} ;
				\node[fill=blue!30] (Y) at (3.5,0) {$Y$};
				\node (Z) at (1,0.5) {$W_1$};
				\node (W) at (2.5, 0.5) {$W_2$};

				\draw [<-,>=latex] (X) -- (Z);
				\draw [<-,>=latex] (X) -- (W);
				
				\draw [->,>=latex] (Z) -- (Y);
				\draw [->,>=latex] (W) -- (Y);
				
				\draw [->,>=latex] (X) -- (Y);
				
			\end{tikzpicture}
			
		\end{subfigure}
		
		\begin{subfigure}{\textwidth}
	\centering
	\caption{Other causal  DAG 2.}
			\begin{tikzpicture}[{black, circle, draw, inner sep=0}]
				\tikzset{nodes={draw,rounded corners},minimum height=0.5cm,minimum width=0.5cm, font=\footnotesize}	
				\tikzset{latent/.append style={white, fill=black}}
				
				\node[fill=red!30] (X) at (0,0) {$X$} ;
				\node[fill=blue!30] (Y) at (3.5,0) {$Y$};
				\node (Z) at (1,0.5) {$W_1$};
				\node (W) at (2.5, 0.5) {$W_2$};
				
				\draw [->,>=latex] (Z) -- (Y);
				\draw [<-,>=latex] (X) -- (W);
				
			\end{tikzpicture}
			
		\end{subfigure}
	\end{minipage}
	
	\caption{Three difference graphs ((c), (h), and (m)), each associated with two pairs of causal DAGs (one pair on the left and one on the right). The two causal DAGs in each pair share the same topological ordering, ensuring that Assumption~\ref{assumption:order} is satisfied. In each subfigure, red and blue vertices represent the cause and effect of interest. In (c), neither the total nor the direct effect are identifiable. In (h), only the total effect is identifiable, while in (m), only the direct effect is identifiable.}
	\label{fig:acyclic}
\end{figure*}

Recently, difference graphs have been introduced as a tool to represent the differences between two SCMs, each corresponding to a distinct population, offering a refined method to visually compare causal variations between these groups. 
\begin{definition}[Difference graphs]
	\label{def:Diff}
	Consider an unordered pair of SCMs ($\mathcal{M}^1=(\mathbb{V},\mathbb{U},\mathbb{F}^1,\Pr^1)$, $\mathcal{M}^2(\mathbb{V},\mathbb{U},\mathbb{F}^2,\Pr^2)$) that respectively induce the two  causal DAGs  $\mathcal{G}^1$ and $\mathcal{G}^2$ and that are respectively compatible with the two distributions  $\Pr^1(\mathbb{v})$ and $\Pr^2(\mathbb{v})$.
	A  \emph{difference graph} $\mathcal{D} = (\mathbb{V}, \mathbb{E}^{{\mid1-2\mid}})$ between $\mathcal{M}^1$ and $\mathcal{M}^2$ is a directed graph where the set of vertices is identical of the set of endogenous variables in $\mathcal{M}^1$ and $\mathcal{M}^2$ and the set of edges is defined as follows: 
	$\mathbb{E}^{\mid1-2\mid} := \{X\rightarrow Y \mid \forall X,Y \in \mathbb{V}$ if the direct effect of $X$ on $Y$ in $\mathcal{M}^1$ is different than the direct effect\footnote{In a linear model, the direct effects are reduce to path coefficients \citep{Wright_1920,Wright_1921} denoted as $\alpha_{x,y}$.} of $X$ on $Y$ in $\mathcal{M}^2$, \ie, if there is a mechanism change~\citep{Tian_2001}  at $Y$ with respect to $X$.
\end{definition}
Note that multiple unordered pairs of SCMs can define the same difference graph.
In Definition~\ref{def:Diff}, difference graphs are defined using SCMs rather than causal DAGs because causal variations may not always be structurally evident by comparing the two causal DAGs, since changes are often parametric rather than structural. 
That said, Figures~\ref{fig:acyclic} and \ref{fig:cyclic},  feature six examples of difference graphs ((c), (h), (m) in Figure~\ref{fig:acyclic} and (c), (f), (k) in Figures~\ref{fig:cyclic})
along with pairs of causal DAGs on both side of each difference graph that can also define the same difference graph (indicating structural differences\footnote{The term "structural differences"  refers to "topological differences between the two causal DAGs".} between the two populations). 
This selection was made intentionally in order to make the concept of difference graphs visually clearer and easier to understand. Nonetheless, it is important to remember that the theoretical results of this paper apply even when the difference between the two populations is not structurally apparent.
However, all those results  are grounded in the following assumption regarding the distributions that are compatible with the difference graph.
\begin{assumption}
	\label{assumption:pos}
	$\Pr^1(\mathbb{v})$ and $\Pr^2(\mathbb{v})$ are positive distributions.
\end{assumption}


As shown in Figure~\ref{fig:cyclic}, generally, there is no inherent necessity to impose similarity constraints on two causal DAGs and thus difference graphs can contain cycles. However, to make the results of this paper  more comprehensible and consistent with recent literature on difference graphs, I initially assume that the two causal DAGs have the same topological ordering. 
\begin{assumption}
	\label{assumption:order}
	$\mathcal{G}^1$ and $\mathcal{G}^2$ share the same topological ordering.
\end{assumption}
This assumption implies that the difference graph is also a DAG~\citep{Wang_2018} and it can hold in many applications, for example, in biological pathways an upstream gene does not generally become a downstream gene in different conditions~\citep{Wang_2018}. However, in some domains, it is possible to imagine scienarios where causal relations can change direction under different conditions. 
For example,  in epidemiology, obesity often leads to the development of type 2 diabetes in the general population, with no direct feedback from diabetes to obesity. However, in a subpopulation of individuals taking certain antipsychotic medications, the onset of type 2 diabetes can lead to significant weight gain due to the medication's side effects, without obesity being the original cause.
Similarly,  in nephrology, 
diabetes can lead to chronic kidney disease, as high blood sugar damages the kidneys in the general population. However, in patients with end-stage kidney disease, impaired kidney function can cause insulin resistance, leading to diabetes, without diabetes further worsening kidney disease.
As a result, for each finding presented in this paper under Assumption~\ref{assumption:order}, I will also provide an alternative result that extends to cases where Assumption~\ref{assumption:order} is not satisfied.




Recently, there has been increasing interest in discovering difference graphs directly from data, without first discovering the two SCMs that define them \citep{Wang_2018, Malik_2024, Bystrova_2024_b}. However, a key question that has yet to be explored, but naturally arises, is whether these graphs can provide conditions for estimating causal changes between two populations using observational data. 
This work considers two types of causal changes: total causal change and direct causal change, also referred to as mechanism change~\citep{Tian_2001}. 
The total causal change of $X$ on $Y$ is defined as the difference between the total effect of $X$ on $Y$ in $\mathcal{M}^1$ and the total effect of $X$ on $Y$ in $\mathcal{M}^2$. 
In a nonparametric setting, the total effects is a specific functional of the interventional distribution  $\probac{Y=y}{\interv{X=x}}$ (where the $do()$ operator represents the intervention) for different values of $x$. However, for simplicity, the total effect is often reffered to as $\probac{Y=y}{\interv{X=x}}$ and
by a slight abuse of notation, in the remainder, the  total effect will  be denoted as $\probac{y}{\interv{x}}$.
Similarly, the direct causal change is the difference between the direct effect of $X$ on $Y$ in $\mathcal{M}^1$ and the direct effect of $X$ on $Y$ in $\mathcal{M}^2$.
In a nonparametric setting, the direct effect can be defined as a specific functional of $\Pr(y\mid do(x), do(\mathbb{S}_{x,y}))$ for different values of $x$, known as the controlled direct effect~\citep{Pearl_2000}, where $\mathbb{S}_{x,y}$ represents the set of all observed variables in the system except $X$ and $Y$\footnote{Another approach to defining a direct effect in a nonparametric setting is through the natural direct effect~\citep{Pearl_2001}, which relies on counterfactual notions rather than the do() operator. However, since it cannot be defined purely using interventional concepts, it is considered beyond the scope of this paper.}. The  controlled direct effect of $X$ on $Y$ is not uniquely defined as it can change with respect to the values of $\mathbb{S}_{x,y}$ (some variables in $\mathbb{S}_{x,y}$ are known as effect modifier). For instance, the direct effect of a pill on thrombosis is likely different for pregnant and non-pregnant women \citep{Pearl_2000}.
However, in a linear SCM, i.e, $\forall f_i \in \mathbb{F}$ $f_i$, is a linear function, the direct effect, obtained through differentiation, corresponds to the familiar path coefficient $\alpha_{x,y}$ \citep{Wright_1920,Wright_1921}, which, unlike the non-parametric controlled direct effect, is uniquely defined. Therefore, in this work, whenever direct effect is discussed, it will refer specifically to linear SCMs.

Since the total causal change of $X$ on $Y$ is a functional of the total effect, $\probac{y}{\interv{x}}$, and the direct causal change is a functional of the direct effect, $\alpha_{x,y}$, the total causal change is identifiable if the total effect is identifiable, and likewise, the direct causal change is identifiable if the direct effect is identifiable\footnote{In some cases, a causal change can be identified even when the corresponding causal effect is not, e.g., if there is no directed path from $X$  to $Y$, the causal change is clearly zero, regardless of whether the causal effect is identifiable.}. Therefore, the focus of the remainder of this paper is to establish the identifiability of both the total and direct effects using difference graphs. I begin by providing a general definition of the identifiability of causal effects using difference graphs, which is adapted from the definition given by \cite{Pearl_2000} for the identifiability of causal effects using causal DAGs.


\begin{definition}[Causal effect identifiability using difference graph]
	\label{def:Identifiability}
	Let $X$ and $Y$ be distinct vertices in a difference graph $\mathcal{D} =
	(\mathbb{V}, \mathbb{E}^{\mid 1-2 \mid})$. The causal effect $Q$ of $X$ on $Y$ is identifiable in  $\mathcal{D}$ if $Q$ is uniquely computable from any observational distributions consistent with $\mathcal{D}$.
\end{definition}

Given a causal DAG, it is possible to identify $\probac{y}{\interv{x}}$ using the  back-door criterion and  to identify $\alpha_{x,y}$ using the  single-door criterion which is only applicable in a linear setting.
\begin{definition}[Back-door criterion, \cite{Pearl_1995}]
	A set of variables $\mathbb{W}$ satisfies the back-door criterion
	relative to the total effect $\probac{y}{\interv{x}}$  in a causal DAG $\mathcal{G}$ if   $\forall W\in \mathbb{W}$, $W\not\in De(X, \mathcal{G})$ and  $\mathbb{W}$ blocks every path between $X$ and $Y$ which contains an arrow into $X$.
\end{definition}

\begin{definition}[Single-door criterion, \cite{Spirtes_1998,Pearl_1998,Pearl_2000}]
	A set of variables $\mathbb{W}$ satisfies the single-door criterion
	relative to the direct effect $\alpha_{x,y}$  in a causal DAG $\mathcal{G}$ if  
	$\forall W\in \mathbb{W}$, $W\not\in De(Y, \mathcal{G})$
	and  $\mathbb{W}$ blocks every path between $X$ and $Y$.
\end{definition}
Whenever a set $\mathbb{W}$ satisfies the back-door criterion\footnote{\textcolor{green}{The back-door criterion is in general not complete. However, when  $X$ is a singleton (as considered in this paper), an adjustment set exists iff there is at least one set that satisfies the back-door criterion~\citep{Perkovic_2018}.}} in a causal DAG, $\probac{y}{\interv{x}}$ is given by the adjsutment formula $\sum_{\mathbb{w}}\probac{y}{x,\mathbb{w}}\proba{\mathbb{w}}$.
Whenever a set $\mathbb{W}$ satisfies the single-door criterion in a causal DAG, $\alpha_{x,y}$ is given by the regression coefficient of $X$ on $Y$ after $ \mathbb{W}$ is partialled out.

Since this work focuses on difference graphs, the first objective is to identify a set that satisfies the back-door criterion across all causal DAGs compatible with a given difference graph when estimating the total effect—if such a set exists, the total effect is said to be identifiable by a \emph{common back-door}. Similarly, the second objective is to find a set that satisfies the single-door criterion across all compatible causal DAGs when estimating the direct effect—if such a set exists, the direct effect is said to be identifiable by a \emph{common single-door}.
However, these goals cannot be acheived by applying the back-door and single-door criteria  directly to difference graphs, even when the difference graph is acyclic. For example, in the difference graph shown in Figure~\ref{fig:acyclic} (m), the sets $\emptyset$, $\{W_1\}$, $\{W_2\}$, and $\{W_1, W_2\}$ all satisfy the back-door criterion for the total effect $\probac{y}{\interv{x}}$. But when examining the pair of causal DAGs in (k) and (l) that are compatible with this difference graph, it becomes clear that $W_1$ is necessary and $W_2$ must be excluded from any set that satisfies the back-door criterion, ruling out the empty set and $\{W_2\}$. Additionally, by looking at another pair of causal DAGs in (n) and (o), it becomes evident that in this case, $\{W_2\}$ is required for the back-door criterion to hold. This illustrates that, in this example, no single set satisfies the back-door criterion for every causal DAG compatible with the difference graph. Similarly, for the difference graph in Figure~\ref{fig:acyclic} (h), the sets $\emptyset$, $\{W_1\}$, $\{W_2\}$, and $\{W_1, W_2\}$ all satisfy the single-door criterion for the direct $\alpha_{x,y}$. However, the pair of causal DAGs shown in (f) and (g), both $W_1$ and $W_2$ are needed in any set that satisfies the single-door criterion, whereas in the pair of causal DAGs shown in (i) and (j), $W_2$ must not be included in any such set. This demonstrates that, in this example, no single set satisfies the single-door criterion for all causal DAGs compatible with the difference graph.

\textcolor{green}{
	In addition to examining identifiability by a common back-door (for the total effect) or by a common single-door (for the direct effect), I also consider cases where the total or direct effect is identifiable because it is evident from the graph that the causal effect does not exist. For example, in a DAG where $Y$ causes $X$, the total and direct effects of $X$ on $Y$ are both identifiable: $\probac{y}{\interv{x}}$ reduces to $\proba{y}$, and $\alpha_{x,y}$ is zero. In such cases,  the total or direct effect are referred to as being  identifiable \textit{trivially}.
}
Now that all necessary terminologies and notations have been established, the two problems that this work aims to address can be formally presented.

\begin{problem}
	Consider a difference graph $\mathcal{D}$  defined by two \emph{unknown} SCMs that respectively induce two \emph{unknown} causal DAGs. Is  $\probac{y}{\interv{x}}$  identifiable in $\mathcal{D}$ \textcolor{green}{trivially or} by a common back-door?
\end{problem}

\begin{problem}
	Consider a difference graph $\mathcal{D}$ defined by two \emph{unknown} linear SCMs that respectively induce two \emph{unknown} causal DAGs. Is  $\alpha_{x,y}$ identifiable in $\mathcal{D}$ \textcolor{green}{trivially or} by a common single-door?
\end{problem}

\section{Identification of total effects using difference graphs in a nonparametric setting}
\label{sec:identification_main_total}

This section first presents conditions for identifying total effects using difference graphs in a nonparametric setting under Assumption~\ref{assumption:order} and then generalizes this result by relaxing the assumption.

\begin{theorem}
	\label{theorem:total_effect}
	Consider a difference graph $\mathcal{D}$ compatible with two different SCMs.
	Under Assumptions~\ref{assumption:hidden}, \ref{assumption:pos}  and \ref{assumption:order},
	$\probac{y}{\interv{x}}$ is  identifiable  in $\mathcal{D}$  \textcolor{green}{trivially or} by a common back-door iff:
	\begin{enumerate}[label=\textbf{A.\arabic*}]
		\item\label{item:theorem:total_effect_1} $Y \in An(X,\mathcal{D})$; or
		\item\label{item:theorem:total_effect_2}  $X \in An(Y,\mathcal{D})$  and  $\forall W\in \mathbb{V}\backslash\{X,Y\}$, $X\in An(W, \mathcal{D})$  or $W\in An(X, \mathcal{D})$.		
	\end{enumerate}
	Furthermore, if Condition~\ref{item:theorem:total_effect_1} is satisfied then $\probac{y}{\interv{x}}=\proba{y}$ and if Condition~\ref{item:theorem:total_effect_2} is satisfied then $\probac{y}{\interv{x}}=\sum_{\mathbb{w}^{an}}\probac{y}{x, \mathbb{w}^{an}}\proba{\mathbb{w}^{an}}$, where $\mathbb{W}^{an}= An(X,\mathcal{D})\backslash\{X\}$.
\end{theorem}
\begin{proof}
	Lemmas~\ref{lemma:total_effect_y_to_x} and \ref{lemma:total_effect_x_to_y} respectively prove that Conditions~\ref{item:theorem:total_effect_1} and ~\ref{item:theorem:total_effect_2} are sufficient for identifiability and Lemmas~\ref{lemma:total_effect_y_to_x} and \ref{lemma:total_effect_x_indep_w} prove that they are also necessary.
\end{proof}

Before delving into the four lemmas that form the foundation of the proof for Theorem~\ref{theorem:total_effect}, it is important to highlight the key implications of the theorem. The theorem provides conditions under which the total effect can be identified in difference graphs by a common back-door. The identification of these total effects hinges on the structural properties of the difference graph.

In particular, when there is no directed path between the variables of interest\textemdash either from $X$ to $Y$ or from $Y$ to $X$\textemdash as seen in the difference graph in Figure~\ref{fig:acyclic} (c), the absence of such a connection creates ambiguity in the directionality of the total effect. This is illustrated by the pair of causal DAGs in (a) and (b), where the association between $X$ and $Y$ reflects the total effect of interest, and by the pair in (d) and (e), where the association between $X$ and $Y$ does not correspond to the total effect being sought. Therefore, the total effect cannot be identified using a common back-door.
This is an essential point, as it highlights that causal reasoning in difference graphs is counterintuitive, with nonexistent directed paths in the difference graph potentially leading to ambiguous causal interpretations. The following lemma builds on this observation and in addition it shows that Condition~\ref{item:theorem:total_effect_1} of the theorem is sufficient.

\begin{lemma}
	\label{lemma:total_effect_y_to_x}
	Under Assumptions~\ref{assumption:hidden}, \ref{assumption:pos} and \ref{assumption:order}, if there exists a directed path from $Y$ to $X$ then $\probac{y}{\interv{x}}=\proba{y}$ and if there exists no directed path from $X$ to $Y$ or from $Y$ to $X$ in $\mathcal{D}$ then $\probac{y}{\interv{x}}$  is not identifiable.
\end{lemma}
\begin{proof}
	If there exists a directed path from $Y$ to $X$ then by Assumption~\ref{assumption:order}, there cannot exists a directed path from $X$ to $Y$, which means that $X$ does not have an effect on $Y$, \ie, $\probac{y}{\interv{x}}=\proba{y}$. 	If there is no directed path from $X$ to $Y$ or from $Y$ to $X$ in $\mathcal{D}$ then there exists a possible  pair of causal DAGs $(\mathcal{G}^1, \mathcal{G}^2)$ defining $\mathcal{D}$ such that  $X\rightarrow Y$ in $\mathcal{G}^1$ and $\mathcal{G}^2$.  In this case $\probac{y}{\interv{x}}\ne \Pr(y)$.	
	At the same time there exists another possible  pair of causal DAGs $(\mathcal{G}^{1'}, \mathcal{G}^{2'})$ defining $\mathcal{D}$ such that $X\leftarrow Y$ in $\mathcal{G}^{1'}$ and $\mathcal{G}^{2'}$. In this case, $\probac{y}{\interv{x}}= \Pr(y)$.
\end{proof}

Condition~\ref{item:theorem:total_effect_2}  of Theorem~\ref{theorem:total_effect} can be applied to identify $\probac{y}{\interv{x}}$ in the difference graph shown in Figure~\ref{fig:acyclic} (h), as $W_1\in An(X, \mathcal{D})$ and $X\in An(W_2, \mathcal{D})$. This implies that for any pair of causal DAGs compatible with the difference graph—such as the pairs depicted in Figures~\ref{fig:acyclic} (f) and (g) or Figures~\ref{fig:acyclic} (i) and (j)—the set $\{W_1\}$ satisfies the back-door criterion. The following lemma illustrates that  Condition~\ref{item:theorem:total_effect_2} is indeed sufficient for identifiability.

\begin{lemma}
	\label{lemma:total_effect_x_to_y}
	Under Assumptions~\ref{assumption:hidden}, \ref{assumption:pos} and \ref{assumption:order}, if there exists a directed path from $X$ to $Y$ and for each vertex $W\in \mathbb{V}\backslash\{X,Y\}$ there exists a directed path from $X$ to $W$ or from  $W$ to $X$ in $\mathcal{D}$, then 
	$\probac{y}{\interv{x}}$  is  identifiable in $\mathcal{D}$ by a common back-door.
\end{lemma}
\begin{proof}
	Consider any pair of causal DAGs ($\mathcal{G}^1$, $\mathcal{G}^2$) that define $\mathcal{D}$. Let $\mathbb{W}^{an}=An(X, \mathcal{D})\backslash\{X,Y\}$ 
	and let $\mathbb{W}^{de}=De(X, \mathcal{D})\backslash\{X,Y\}$. 
	For any $W\in \mathbb{W}^{an}$, $W$ cannot be in $\mathbb{W}^{de}$, otherwise $\mathcal{D}$ would not be acyclic and Assumption~\ref{assumption:order} would  not be satisfied. 
	By the lemma, for any $W\in \mathbb{V}\backslash\{X,Y\}$ if $W\notin \mathbb{W}^{an}$ then  $W\in \mathbb{W}^{de}$. 
	It follows that $\forall W\in \mathbb{W}^{an}$, $W$ has to be an ancestor of $X$ in at least one of the two causal DAGs $\mathcal{G}^1$ and  $\mathcal{G}^2$ and $W$ is not a descendent of $X$ in any of the two causal DAGs.
	Thus $\mathbb{W}^{an}$ has to block all paths from $X$ to $Y$ which contains an arrow into $X$ and at the same time $\mathbb{W}^{an}$ cannot contain any descendant of $X$. Therefore,  $\mathbb{W}^{an}$ satisfies the back-door criterion~\citep{Pearl_1995} in $\mathcal{G}^1$ and in $\mathcal{G}^2$ and by \citet[Theorem 1]{Pearl_1995} the total effect $\probac{y}{\interv{x}}$ in $\mathcal{M}^1$ and in $\mathcal{M}^2$ is given by $\sum_{\mathbb{w}^{an}}\probac{y}{x, \mathbb{w}^{an}}\proba{\mathbb{w}^{an}}$.
\end{proof}


Theorem~\ref{theorem:total_effect}  indicates non-identifiability if Conditions~\ref{item:theorem:total_effect_1} and \ref{item:theorem:total_effect_2} are not met. For instance, Theorem~\ref{theorem:total_effect} shows that $\probac{y}{\interv{x}}$ is not identifiable in the difference graph  in Figure~\ref{fig:acyclic} (m). This is evident by considering two possible pairs of causal DAGs compatible with the difference graph: Figures~\ref{fig:acyclic} (k) and (l), and Figures~\ref{fig:acyclic} (n) and (m). In the first pair, $\{W_1\}$ satisfies the back-door criterion, but $\{W_1, W_2\}$ does not. In contrast, in the second pair, only $\{W_1, W_2\}$ satisfies the back-door criterion.
Lemma~\ref{lemma:total_effect_y_to_x} provided one part of the proof for the necessity of these two conditions, and the following lemma completes the other half, thus concluding the proof of the theorem.

\begin{lemma}
	\label{lemma:total_effect_x_indep_w}
	Under Assumptions~\ref{assumption:hidden}, \ref{assumption:pos} and \ref{assumption:order}, if there exists a directed path from $X$ to $Y$ and there exists  a vertex $W\in \mathbb{V}\backslash\{X,Y\} $ for which there is no directed path from $X$ to $W$ or from $W$ to $X$ in $\mathcal{D}$ then $\probac{y}{\interv{x}}$  is not identifiable in $\mathcal{D}$ by a common back-door.
\end{lemma}
\begin{proof}
	\textcolor{green}{
		Suppose there is a directed path from $X$ to $Y$ in $\mathcal{D}$. 
		Now, consider a vertex $W$ for which there is no directed path from or to $X$ in $\mathcal{D}$. It follows that in any pair of causal DAGs $\mathcal{G}^1$ and $\mathcal{G}^2$, $W$ can be either independent of $X$ or a descendant of $X$ or a parent of $X$ or an ancestor, but not a parent, of $X$. 
		Let us focus on the second and third cases.
		In the second case, $W$ cannot be included in any set that satisfies the back-door criterion in every causal DAG compatible with $\mathcal{D}$, since $W$ is a descendant of $X$. In the third case, $W$ cannot be a descendant of $Y$; otherwise, a cycle would form involving $X$, $W$, and $Y$. Therefore, in at least one causal DAG compatible with $\mathcal{D}$, $W$ must be a parent of $Y$, meaning $W$ must be included in any set that satisfies the back-door criterion in every causal DAG compatible with $\mathcal{D}$. This third case contradicts the second case, making it impossible to find a set that satisfies the back-door criterion in all compatible causal DAGs.
	}
\end{proof}

The remaining part of this section  investigates the case where the two causal  DAGs induced by the two SCMs that define the difference graph do not share the same topological ordering (i.e., Assumption~\ref{assumption:order} is not satisfied). In this case the difference graph can contain cycles.

It is clear that if Assumption~\ref{assumption:order} is not satisfied, the first part of Lemma~\ref{lemma:total_effect_y_to_x} does not hold. An example of this can be seen in Figure~\ref{fig:cyclic} (c), where the difference graph contains a directed path from $Y$ to $X$, yet the total effect cannot be identified due to the simultaneous existence of a directed path from $X$ to $Y$. Additionally, it is evident that if Assumption~\ref{assumption:order} is not met, Lemma~\ref{lemma:total_effect_x_to_y} as a whole will fail. Furthermore, even though the second part of Lemma~\ref{lemma:total_effect_y_to_x} and Lemma~\ref{lemma:total_effect_x_indep_w} remain valid individually, together they do not provide necessary conditions.
However, as shown in the following theorem, minor modifications to  these lemmas would allow acheiving correctness. The proof of this new theorem will simply highlight how the proofs of the previous lemmas should be modified.

\begin{theorem}
	\label{theorem:total_effect_cycle}
	Consider a difference graph $\mathcal{D}$ compatible with two different SCMs.
	Under Assumptions~\ref{assumption:hidden} and \ref{assumption:pos},
	$\probac{y}{\interv{x}}$ is  identifiable in $\mathcal{D}$  \textcolor{green}{trivially or} by a common back-door  iff 
	\begin{enumerate}[label=\textbf{B.\arabic*}]
		\item \label{item:theorem:total_effect_cycle_1}  Condition~\ref{item:theorem:total_effect_1} is satisfied and  $X\not\in An(Y, \mathcal{D})$; or
		\item  \label{item:theorem:total_effect_cycle_2}  Condition~\ref{item:theorem:total_effect_2}  is satisfied and
		$An(X, \mathcal{D})\cap De(X, \mathcal{D})=\{X\}$.
	\end{enumerate}
	Furthermore, if Condition~\ref{item:theorem:total_effect_cycle_1} is satisfied then $\probac{y}{\interv{x}}=\proba{y}$ and if Condition~\ref{item:theorem:total_effect_cycle_2} is satisfied then $\probac{y}{\interv{x}}=\sum_{\mathbb{w}^{an}}\probac{y}{x, \mathbb{w}^{an}}\proba{\mathbb{w}^{an}}$, where $\mathbb{W}^{an}= An(X,\mathcal{D})\backslash\{X\}$.
\end{theorem}

\begin{proof}
	The proof is very similar to the proof of Theorem~\ref{theorem:total_effect} with sime minor difference. First, notice  that replacing Assumption~\ref{assumption:order} with the condition that there is no directed path from $X$ to $Y$ in the proof of Lemma~\ref{lemma:total_effect_y_to_x} would maintain correctness. Second  notice  that replacing Assumption~\ref{assumption:order} with the condition that there is no cycle containing $X$ and another vertex in the proof of Lemma~\ref{lemma:total_effect_x_to_y} would maintain correctness. 
	Finally, to extend Lemma~\ref{lemma:total_effect_x_indep_w}, I add the condition that the total effect is not identifiable by a common back-door if a cycle exists on $X$. Combining this with the second part of Lemma~\ref{lemma:total_effect_y_to_x} establishes that the conditions are necessary. This can be easily shown because if there is a cycle in $\mathcal{D}$ involving $X$ and another vertex $W$ (where $W \neq Y$), then there must exist a causal DAG in which $W$ is a parent of $X$, as well as another causal DAG in which $W$ is a descendant of $X$. In both of these graphs, $W$ can also be a parent of $Y$. This implies that in the first DAG, $W$ must be included in any set satisfying the back-door criterion, while in the second DAG, $W$ cannot be included in such a set. Therefore, no set can satisfy the back-door criterion for all causal DAGs compatible with $\mathcal{D}$. Moreover, if the cycle only involves $X$ and $Y$, the total effect is also not identifiable for the same reason the total effect is not identifiable when there is no directed path from $X$ to $Y$ or from $Y$ to $X$.
\end{proof}

Theorem~\ref{theorem:total_effect_cycle} shows that $\probac{y}{\interv{x}}$ is not identifiable in the difference graphs depicted in Figures~\ref{fig:cyclic} (c) and (k), but it is identifiable in the difference graph shown in Figure~\ref{fig:cyclic} (f).

\begin{figure*}[t!]
	\centering
	\begin{minipage}{.4\textwidth}
		\begin{subfigure}{\textwidth}
	\centering
	\caption{Causal  DAG 1.}
			\begin{tikzpicture}[{black, circle, draw, inner sep=0}]
				\tikzset{nodes={draw,rounded corners},minimum height=0.5cm,minimum width=0.5cm, font=\footnotesize}	
				\tikzset{latent/.append style={white, fill=black}}
				
				\node[fill=red!30] (X) at (0,0) {$X$} ;
				\node[fill=blue!30] (Y) at (3.5,0) {$Y$};
				
				\draw [->,>=latex] (X) -- (Y);
				
			\end{tikzpicture}			
			
		\end{subfigure}
		
		\begin{subfigure}{\textwidth}
	\centering
	\caption{Causal  DAG 2.}
			\begin{tikzpicture}[{black, circle, draw, inner sep=0}]
				\tikzset{nodes={draw,rounded corners},minimum height=0.5cm,minimum width=0.5cm, font=\footnotesize}	
				\tikzset{latent/.append style={white, fill=black}}
				
				\node[fill=red!30] (X) at (0,0) {$X$} ;
				\node[fill=blue!30] (Y) at (3.5,0) {$Y$};
				
				\draw [<-,>=latex] (X) -- (Y);				
				
			\end{tikzpicture}
			
		\end{subfigure}
	\end{minipage}
	\hfill
	\vline
	\hfill
	\begin{minipage}{.4\textwidth}
		\begin{subfigure}{\textwidth}
	\centering
	\caption{Difference graph.}
			\begin{tikzpicture}[{black, circle, draw, inner sep=0}]
				\tikzset{nodes={draw,rounded corners},minimum height=0.5cm,minimum width=0.5cm, font=\footnotesize}	
				\tikzset{latent/.append style={white, fill=black}}
				
				\node[fill=red!30] (X) at (0,0) {$X$} ;
				\node[fill=blue!30] (Y) at (3.5,0) {$Y$};

				\begin{scope}[transform canvas={yshift=-.25em}]
					\draw [->,>=latex] (X) -- (Y);
				\end{scope}
				\begin{scope}[transform canvas={yshift=.25em}]
					\draw [<-,>=latex] (X) -- (Y);
				\end{scope}			
				
			\end{tikzpicture}
			
		\end{subfigure}	
	\end{minipage}
	
	\hrulefill
	
	%
	%
	\begin{minipage}{.3\textwidth}
		\begin{subfigure}{\textwidth}
	\centering
	\caption{Causal  DAG 1.}
			\begin{tikzpicture}[{black, circle, draw, inner sep=0}]
				\tikzset{nodes={draw,rounded corners},minimum height=0.5cm,minimum width=0.5cm, font=\footnotesize}	
				\tikzset{latent/.append style={white, fill=black}}
				
				\node[fill=red!30] (X) at (0,0) {$X$} ;
				\node[fill=blue!30] (Y) at (3.5,0) {$Y$};
				\node (Z) at (1,0.5) {$W_1$};
				\node (W) at (2.5, 0.5) {$W_2$};
				
				\draw [->,>=latex] (Z) -- (X);
				\draw [->,>=latex] (Z) -- (Y);
				\draw [->,>=latex] (X) -- (W);
				\draw [->,>=latex] (W) -- (Y);
				
				\draw [->,>=latex] (X) -- (Y);
				
			\end{tikzpicture}			
			
		\end{subfigure}
		
		\begin{subfigure}{\textwidth}
	\centering
	\caption{Causal  DAG 2.}
			\begin{tikzpicture}[{black, circle, draw, inner sep=0}]
				\tikzset{nodes={draw,rounded corners},minimum height=0.5cm,minimum width=0.5cm, font=\footnotesize}	
				\tikzset{latent/.append style={white, fill=black}}
				
				\node[fill=red!30] (X) at (0,0) {$X$} ;
				\node[fill=blue!30] (Y) at (3.5,0) {$Y$};
				\node (Z) at (1,0.5) {$W_1$};
				\node (W) at (2.5, 0.5) {$W_2$};
				
				\draw [->,>=latex] (Z) -- (Y);
				\draw [<-,>=latex] (W) -- (Y);

			\end{tikzpicture}
			
		\end{subfigure}
	\end{minipage}
	\hfill
	\vline
	\hfill
	\begin{minipage}{.3\textwidth}
		\begin{subfigure}{\textwidth}
	\centering
	\caption{Difference graph.}
			\begin{tikzpicture}[{black, circle, draw, inner sep=0}]
				\tikzset{nodes={draw,rounded corners},minimum height=0.5cm,minimum width=0.5cm, font=\footnotesize}	
				\tikzset{latent/.append style={white, fill=black}}
				
				\node[fill=red!30] (X) at (0,0) {$X$} ;
				\node[fill=blue!30] (Y) at (3.5,0) {$Y$};
				\node (Z) at (1,0.5) {$W_1$};
				\node (W) at (2.5, 0.5) {$W_2$};

				\draw [->,>=latex] (Z) -- (X);

				\draw [->,>=latex] (X) -- (W);
				
				\begin{scope}[transform canvas={yshift=-.25em}]
					\draw [->,>=latex] (W) -- (Y);
				\end{scope}
				\begin{scope}[transform canvas={yshift=.25em}]
					\draw [<-,>=latex] (W) -- (Y);
				\end{scope}			
				
				\draw [->,>=latex] (X) -- (Y);
			\end{tikzpicture}
			
		\end{subfigure}
	\end{minipage}
	\hfill
	\vline
	\hfill
	\begin{minipage}{.3\textwidth}
		\begin{subfigure}{\textwidth}
	\centering
	\caption{Other causal  DAG 1.}
			\begin{tikzpicture}[{black, circle, draw, inner sep=0}]
				\tikzset{nodes={draw,rounded corners},minimum height=0.5cm,minimum width=0.5cm, font=\footnotesize}	
				\tikzset{latent/.append style={white, fill=black}}
				
				\node[fill=red!30] (X) at (0,0) {$X$} ;
				\node[fill=blue!30] (Y) at (3.5,0) {$Y$};
				\node (Z) at (1,0.5) {$W_1$};
				\node (W) at (2.5, 0.5) {$W_2$};
				
				\draw [->,>=latex] (Z) -- (X);
				\draw [->,>=latex] (Z) -- (Y);
				
				\draw [->,>=latex] (X) -- (W);
				\draw [<-,>=latex] (W) -- (Y);
				
				\draw [->,>=latex] (X) -- (Y);
				
			\end{tikzpicture}
			
		\end{subfigure}
		
		\begin{subfigure}{\textwidth}
	\centering
	\caption{Other causal  DAG 2.}
			\begin{tikzpicture}[{black, circle, draw, inner sep=0}]
				\tikzset{nodes={draw,rounded corners},minimum height=0.5cm,minimum width=0.5cm, font=\footnotesize}	
				\tikzset{latent/.append style={white, fill=black}}
				
				\node[fill=red!30] (X) at (0,0) {$X$} ;
				\node[fill=blue!30] (Y) at (3.5,0) {$Y$};
				\node (Z) at (1,0.5) {$W_1$};
				\node (W) at (2.5, 0.5) {$W_2$};
				
				\draw [->,>=latex] (Z) -- (Y);
				\draw [->,>=latex] (W) -- (Y);
			\end{tikzpicture}
			
		\end{subfigure}
	\end{minipage}
	
	\hrulefill
	
	%
	%
	%
	%
	%
	\begin{minipage}{.3\textwidth}
		\begin{subfigure}{\textwidth}
	\centering
	\caption{Causal  DAG 1.}
			\begin{tikzpicture}[{black, circle, draw, inner sep=0}]
				\tikzset{nodes={draw,rounded corners},minimum height=0.5cm,minimum width=0.5cm, font=\footnotesize}	
				\tikzset{latent/.append style={white, fill=black}}
				
				\node[fill=red!30] (X) at (0,0) {$X$} ;
				\node[fill=blue!30] (Y) at (3.5,0) {$Y$};
				\node (Z) at (1,0.5) {$W_1$};
				\node (W) at (2.5, 0.5) {$W_2$};
				
				\draw [->,>=latex] (Z) -- (X);
				\draw [->,>=latex] (Z) -- (Y);
				
				\draw [<-,>=latex] (X) -- (W);
				\draw [->,>=latex] (W) -- (Y);
				
				\draw [->,>=latex] (X) -- (Y);
				
			\end{tikzpicture}			
			
		\end{subfigure}
		
		\begin{subfigure}{\textwidth}
	\centering
	\caption{Causal  DAG 2.}
			\begin{tikzpicture}[{black, circle, draw, inner sep=0}]
				\tikzset{nodes={draw,rounded corners},minimum height=0.5cm,minimum width=0.5cm, font=\footnotesize}	
				\tikzset{latent/.append style={white, fill=black}}
				
				\node[fill=red!30] (X) at (0,0) {$X$} ;
				\node[fill=blue!30] (Y) at (3.5,0) {$Y$};
				\node (Z) at (1,0.5) {$W_1$};
				\node (W) at (2.5, 0.5) {$W_2$};
				
				\draw [->,>=latex] (Z) -- (Y);
				
				\draw [->,>=latex] (X) -- (W);

			\end{tikzpicture}
			
		\end{subfigure}
	\end{minipage}
	\hfill
	\vline
	\hfill
	\begin{minipage}{.3\textwidth}
		\begin{subfigure}{\textwidth}
	\centering
	\caption{Difference graph.}
			\begin{tikzpicture}[{black, circle, draw, inner sep=0}]
				\tikzset{nodes={draw,rounded corners},minimum height=0.5cm,minimum width=0.5cm, font=\footnotesize}	
				\tikzset{latent/.append style={white, fill=black}}
				
				\node[fill=red!30] (X) at (0,0) {$X$} ;
				\node[fill=blue!30] (Y) at (3.5,0) {$Y$};
				\node (Z) at (1,0.5) {$W_1$};
				\node (W) at (2.5, 0.5) {$W_2$};

				\draw [->,>=latex] (Z) -- (X);
				\begin{scope}[transform canvas={yshift=-.25em}]
					\draw [->,>=latex] (X) -- (W);
				\end{scope}
				\begin{scope}[transform canvas={yshift=.15em}]
					\draw [<-,>=latex] (X) -- (W);
				\end{scope}			
				
				\draw [->,>=latex] (W) -- (Y);
				
				\draw [->,>=latex] (X) -- (Y);
			\end{tikzpicture}
			
		\end{subfigure}
	\end{minipage}
	\hfill
	\vline
	\hfill
	\begin{minipage}{.3\textwidth}
		\begin{subfigure}{\textwidth}
	\centering
	\caption{Other causal  DAG 1.}
			\begin{tikzpicture}[{black, circle, draw, inner sep=0}]
				\tikzset{nodes={draw,rounded corners},minimum height=0.5cm,minimum width=0.5cm, font=\footnotesize}	
				\tikzset{latent/.append style={white, fill=black}}
				
				\node[fill=red!30] (X) at (0,0) {$X$} ;
				\node[fill=blue!30] (Y) at (3.5,0) {$Y$};
				\node (Z) at (1,0.5) {$W_1$};
				\node (W) at (2.5, 0.5) {$W_2$};
				
				\draw [->,>=latex] (Z) -- (X);
				\draw [->,>=latex] (Z) -- (Y);
				
				\draw [->,>=latex] (X) -- (W);
				\draw [->,>=latex] (W) -- (Y);
				
				\draw [->,>=latex] (X) -- (Y);
				
			\end{tikzpicture}
			
		\end{subfigure}
		
		\begin{subfigure}{\textwidth}
	\centering
	\caption{Other causal  DAG 2.}
			\begin{tikzpicture}[{black, circle, draw, inner sep=0}]
				\tikzset{nodes={draw,rounded corners},minimum height=0.5cm,minimum width=0.5cm, font=\footnotesize}	
				\tikzset{latent/.append style={white, fill=black}}
				
				\node[fill=red!30] (X) at (0,0) {$X$} ;
				\node[fill=blue!30] (Y) at (3.5,0) {$Y$};
				\node (Z) at (1,0.5) {$W_1$};
				\node (W) at (2.5, 0.5) {$W_2$};
				
				\draw [->,>=latex] (Z) -- (Y);
				
				\draw [<-,>=latex] (X) -- (W);
			\end{tikzpicture}
			
		\end{subfigure}
	\end{minipage}
	
	\caption{Three difference graphs ((c), (f), and (k)), each associated with two pairs of causal DAGs (one pair on the left and one on the right). The two causal DAGs in each pair do not share the same topological ordering, therefore Assumption~\ref{assumption:order} is not satisfied. In each subfigure, the red and blue vertices represent the cause and effect of interest, respectively. In the first difference graph (c), neither the total nor the direct effect is identifiable. In the second graph (f), only the total effect is identifiable, while in the third graph (k), only the direct effect is identifiable.}
	\label{fig:cyclic}
\end{figure*}

\paragraph{Small simulation study for total causal changes.}
\textcolor{green}{
	The adjustment set provided by Theorems~\ref{theorem:total_effect} and~\ref{theorem:total_effect_cycle} is not the only possible adjustment set that could be selected using the back-door criterion if the true pair of DAGs were known. 
	However, as shown in Figure~\ref{fig:sim} (a), in practice, the set used in Theorems~\ref{theorem:total_effect} and~\ref{theorem:total_effect_cycle} performs as well as  the  set consisting  of the parents of the exposure and  outperforms the minimal  set derived from the back-door criterion. 
}

\section{Mediation analysis using difference graphs in a linear setting}
\label{sec:identification_main_direct}
This section  presents conditions for identifying the direct effect using a difference graph in a linear setting and as in the previous section, it starts by given conditions under Assumption~\ref{assumption:order} and  then generalizes this result by relaxing the assumption.

\begin{theorem}
	\label{theorem:direct_effect}
	Consider a difference graph $\mathcal{D}$ compatible with two different linear SCMs.
	Under Assumptions~\ref{assumption:hidden}, \ref{assumption:pos} and \ref{assumption:order},
	$\alpha_{x,y}$ is  identifiable in $\mathcal{D}$  \textcolor{green}{trivially or} by a common single-door iff:
	\begin{enumerate}[label=\textbf{C.\arabic*}]
		\item \label{item:theorem:direct_effect_1} $Y \in An(X,\mathcal{D})$; or
		\item  \label{item:theorem:direct_effect_2}  $X \in An(Y,\mathcal{D})$  and  $\forall W\in \mathbb{V}\backslash\{X,Y\}$, $Y\in An(W, \mathcal{D})$  or $W\in An(Y, \mathcal{D})$.		
	\end{enumerate}
	Furthermore, if Condition~\ref{item:theorem:direct_effect_1} is satisfied then $\alpha_{x,y}=0$ and if Condition~\ref{item:theorem:direct_effect_2} is satisfied then $\alpha_{x,y}=r_{YX. \mathbb{W}^{an}}$ where $r_{YX. \mathbb{W}^{an}}$ is the  regression coefficient of $X$ on $Y$ after $ \mathbb{W}^{an}$ is partialled out.
\end{theorem}
\begin{proof}
	Lemmas~\ref{lemma:direct_effect_y_to_x} and \ref{lemma:direct_effect_x_to_y} respectively prove that Conditions~\ref{item:theorem:direct_effect_1} and \ref{item:theorem:direct_effect_2} are sufficient for identifiability and Lemmas~\ref{lemma:direct_effect_y_to_x} and \ref{lemma:direct_effect_y_indep_w} prove that they are also necessary.
\end{proof}

\textcolor{green}{
	Theorem~\ref{theorem:direct_effect} establishes that the direct effect $\alpha_{x,y}$ is identifiable by a  common single-door in the difference graph presented in Figure~\ref{fig:acyclic} (m), but it is not identifiable by a  common single-door  in the difference graphs shown in Figures~\ref{fig:acyclic} (c) and (h). The same reasoning applied in the previous section for the non-identifiability of the total effect in Figure~\ref{fig:acyclic} (c) also applies here for the direct effect. For the case of Figure~\ref{fig:acyclic} (m), it becomes clear why the direct effect is not identifiable when considering the two pairs of causal DAGs compatible with the difference graph, shown in Figures (f) and (g), and Figures (i) and (j). In the first pair, $\{W_1, W_2\}$ is the only set that satisfies the single-door criterion. However, in the second pair, $W_2$ cannot be included in any set that satisfies the single-door criterion.
	I now move on to present the lemmas required for the proof of the theorem, noting that their proofs follow the same reasoning as the lemmas introduced in the previous section.
}

\begin{lemma}
	\label{lemma:direct_effect_y_to_x}
	Under Assumptions~\ref{assumption:hidden}, \ref{assumption:pos} and \ref{assumption:order}, if there exists a directed path from $Y$ to $X$ then $\alpha_{x,y}=0$   and if there exists no directed path from $X$ to $Y$ or from $Y$ to $X$ in $\mathcal{D}$ then $\alpha_{x,y}$  is not identifiable.
\end{lemma}
\begin{proof}
	The proof follows directly from the proof of Lemma~\ref{lemma:total_effect_y_to_x}. 
\end{proof}

\begin{lemma}
	\label{lemma:direct_effect_x_to_y}
	Under Assumptions~\ref{assumption:hidden}, \ref{assumption:pos} and \ref{assumption:order}, if there exists a directed path from $X$ to $Y$ and for each vertex $W\in \mathbb{V}\backslash\{X,Y\}$ there exists a directed path from $Y$ to $W$ or  from  $W$ to $Y$  in $\mathcal{D}$, then 
	$\alpha_{x,y}$  is  identifiable in $\mathcal{D}$ by a common single-door.
\end{lemma}
\begin{proof}
	Consider any pair of causal DAGs ($\mathcal{G}^1$, $\mathcal{G}^2$) that define $\mathcal{D}$. Let
	$\mathbb{W}^{an}=An(Y, \mathcal{D})\backslash\{X,Y\}$ 
	and let $\mathbb{W}^{de}=De(Y, \mathcal{D})\backslash\{X,Y\}$ . 	
	For any $W\in \mathbb{W}^{an}$, $W$ cannot be in $\mathbb{W}^{de}$, otherwise $\mathcal{D}$ would not be acyclic and Assumption~\ref{assumption:order} would  not be satisfied. 
	By the lemma, for any $W\in \mathbb{V}\backslash\{X,Y\}$ if $W\notin \mathbb{W}^{an}$ then  $W\in \mathbb{W}^{de}$. 
	It follows that $\forall W\in \mathbb{W}^{an}$, $W$ has to be an ancestor of $Y$ in at least one of the two causal DAGs $\mathcal{G}^1$ and  $\mathcal{G}^2$ and $W$ is not a descendent of $Y$ in any of the two causal DAGs.
	Thus $\mathbb{W}^{an}$ has to block all paths between $X$ and $Y$ and at the same time $\mathbb{W}^{an}$ cannot contain any descendant of $Y$. Therefore, $\mathbb{W}^{an}$ satisfies the single-door criterion~\citep{Spirtes_1998,Pearl_1998} in $\mathcal{G}^1$ and in $\mathcal{G}^2$ and by \citet[Theorem 6]{Pearl_1998} the direct effect $\alpha_{x,y}$ in $\mathcal{M}^1$ and in $\mathcal{M}^2$ is given by $r_{Y,X.\mathbb{W}^{an}}$.
\end{proof}

\begin{lemma}
	\label{lemma:direct_effect_y_indep_w}
	Under Assumptions~\ref{assumption:hidden}, \ref{assumption:pos} and \ref{assumption:order}, if there exists a directed path from $X$ to $Y$ and there exists  a vertex $W\in \mathbb{V}\backslash\{X,Y\} $ for which there is no directed path from $Y$ to $W$ or from $W$ to $Y$   then $\alpha_{x,y}$  is not identifiable in $\mathcal{D}$ by a common single-door.
\end{lemma}
\begin{proof}
	\textcolor{green}{Suppose there is a directed path from $X$ to $Y$. 
		Now, consider a vertex $W$ for which there is no directed path from or to $Y$ in $\mathcal{D}$. It follows that in any pair of causal DAGs $\mathcal{G}^1$ and $\mathcal{G}^2$, $W$ can be either independent of $Y$ or a descendant of $Y$ or a parent of $Y$ or an ancestor, but not a parent, of $Y$. 
		Let us focus on the second and third cases.
		In the second case, $W$ cannot be included in any set that satisfies the single-door criterion in every causal DAG compatible with $\mathcal{D}$, since $W$ is a descendant of $Y$. In the third case, it is possible to imagine that thre exists a causal DAG compatible with $\mathcal{D}$ where $W$ is either a parent of a child of $X$. Therefore, in this causal DAG, $W$ must be included in any set that satisfies the single-door criterion in every causal DAG compatible with $\mathcal{D}$. This third case contradicts the second case, making it impossible to find a set that satisfies the single-door criterion in all compatible causal DAGs.
	}
\end{proof}

The following theorem extends Theorem~\ref{theorem:direct_effect} to the case where the two causal DAGs, derived from the two SCMs that form the difference graph, do not have the same topological ordering.

\begin{theorem}
	\label{theorem:direct_effect_cycle}
	Consider a difference graph $\mathcal{D}$ compatible with two different linear SCMs.
	Under Assumptions~\ref{assumption:hidden} and \ref{assumption:pos},
	$\alpha_{x,y}$ is  identifiable  in $\mathcal{D}$  \textcolor{green}{trivially or} by a common single-door iff:
	\begin{enumerate}[label=\textbf{D.\arabic*}]
		\item \label{item:theorem:direct_effect_1_cycle}  Condition~\ref{item:theorem:direct_effect_1} is satisfied and  $X\not\in An(Y, \mathcal{D})$; or
		\item  \label{item:theorem:direct_effect_2_cycle}  Condition~\ref{item:theorem:direct_effect_2}  is satisfied and
		$An(Y, \mathcal{D})\cap De(Y, \mathcal{D})=\{Y\}$.
	\end{enumerate}
	Furthermore, if Condition~\ref{item:theorem:direct_effect_1_cycle} is satisfied then $\alpha_{x,y}=0$ and if Condition~\ref{item:theorem:direct_effect_2_cycle} is satisfied then $\alpha_{x,y}=r_{YX. \mathbb{W}^{an}}$ where $r_{YX. \mathbb{W}^{an}}$ is the  regression coefficient of $X$ on $Y$ after $ \mathbb{W}^{an}$ is partialled~out.
\end{theorem}
\begin{proof}
	The proof closely follows the structure of the proof for Theorem~\ref{theorem:direct_effect}, with a few key modifications. First, note that replacing Assumption~\ref{assumption:order} with the condition that there is no directed path from $X$ to $Y$ in the proof of Lemma~\ref{lemma:direct_effect_y_to_x} preserves its correctness. Second, replacing Assumption~\ref{assumption:order} with the condition that no cycle contains both $Y$ and another vertex in the proof of Lemma~\ref{lemma:direct_effect_x_to_y} also maintains its validity. 
	To extend Lemma~\ref{lemma:direct_effect_y_indep_w}, I introduce the condition that the direct effect cannot be identified using a common single-door criterion if there is a cycle involving $Y$. Combining this new condition with the second part of Lemma~\ref{lemma:direct_effect_y_to_x} demonstrates that these conditions of the Theorem are necessary. This can be shown as follows: if a cycle exists in $\mathcal{D}$ involving $Y$ and another vertex $W$ (where $W \neq X$), then there must be at least one causal DAG where $W$ is a parent of $Y$ and another causal DAG where $W$ is a descendant of $Y$. In both of these DAGs, $W$ can also act as either a parent or descendant of $X$. This implies that in the first DAG, $W$ must be included in any set that satisfies the single-door criterion, while in the second DAG, $W$ cannot be included in such a set. As a result, no single set can satisfy the single-door criterion across all causal DAGs compatible with $\mathcal{D}$. Furthermore, if the cycle only involves $X$ and $Y$, the direct effect is also not identifiable for the same reason it is not identifiable when there is no directed path between $X$ and $Y$.
\end{proof}

Theorem~\ref{theorem:direct_effect_cycle} indicates that the direct effect $\alpha_{x,y}$ is not identifiable in the difference graphs shown in Figures~\ref{fig:cyclic} (c) and (f), but it is identifiable in the difference graph depicted in Figure~\ref{fig:cyclic} (k).

\paragraph{Small simulation study for direct causal  changes.}
\textcolor{green}{
	The adjustment set provided by Theorems~\ref{theorem:direct_effect} and~\ref{theorem:direct_effect_cycle} is not the only possible adjustment set that could be selected using the single-door criterion if the true pair of DAGs were known. 
	So similarly, to Section~\ref{sec:identification_main_total},  in Figure~\ref{fig:sim} (b), the set used in Theorems~\ref{theorem:direct_effect} and~\ref{theorem:direct_effect_cycle} is compared with the set consisting of the parents of the outcome and the minimal   set derived from the single-door criterion. In this case, the set used in Theorems~\ref{theorem:direct_effect} and \ref{theorem:direct_effect_cycle} performs slightly less well than the set of parents and better than the minimal set.
}

\begin{figure}
	\begin{minipage}{0.47\textwidth}
		\begin{subfigure}{\textwidth}
				\caption{Total causal change.}
			\begin{tikzpicture}
				\pgfplotsset{
					error bars/.cd,
					x dir=none,
					y dir=both, y explicit,
				}
				\begin{axis}[
					footnotesize,
					width=7.5cm,
					height=2.6cm,
					tickpos=left,
					ytick align=inside,
					enlarge x limits=false,
					xticklabels={,,,Parents,,,, Minimal,,,, Theorem~\ref{theorem:total_effect}\&\ref{theorem:total_effect_cycle}},
					xmin=-0.1,
					xmax=0.5,
					ymin=-0.04,
					ymax=0.15,
					xtick style={draw=none},
					ymajorgrids=true,
					grid style=dashed,
					ytick={0, 0.1},
					legend columns=-1,
					]
					\addplot+ [blue, mark options={fill=blue}, mark=*] table [x expr={\thisrow{x} 0.0}, y=y, y error=ey]{
						x y ey
						0 0.04647007131037813 0.04316960227683745
					};
					\addplot+ [olive, mark options={fill=olive}, mark=*] table [x expr={\thisrow{x} 0.2}, y=y, y error=ey] {
						x y ey
						0 0.07209403564291963 0.12053060333983218
					};
					\addplot+ [purple, mark options={fill=purple}, mark=*] table [x expr={\thisrow{x} 0.4}, y=y, y error=ey] {
						x y ey
						0  0.045397292100194 0.04281655358393379
					};
					\coordinate (top) at (rel axis cs:0,1);
					\coordinate (bot) at (rel axis cs:1,0);
				\end{axis}
				\path (top-|current bounding box.west)-- 
				node[anchor=south,rotate=90] {Error} 
				(bot-|current bounding box.west);
				
			\end{tikzpicture}
		\end{subfigure}
	\end{minipage}
	\hfill 
	\begin{minipage}{0.47\textwidth}
		\begin{subfigure}{\textwidth}
	\caption{Direct causal change.}
			
			\begin{tikzpicture}
				\pgfplotsset{
					error bars/.cd,
					x dir=none,
					y dir=both, y explicit,
				}
				\begin{axis}[
					footnotesize,
					width=7.5cm,
					height=2.6cm,
					tickpos=left,
					ytick align=inside,
					enlarge x limits=false,
					xticklabels={,,,Parents,,,, Minimal,,,, Theorem~\ref{theorem:direct_effect}\&\ref{theorem:direct_effect_cycle}},
					xmin=-0.1,
					xmax=0.5,
					ymin=-0.04,
					ymax=0.15,
					xtick style={draw=none},
					ymajorgrids=true,
					grid style=dashed,
					ytick={0, 0.1},
					legend columns=-1,
					]
					\addplot+ [blue, mark options={fill=blue}, mark=*] table [x expr={\thisrow{x} 0.0}, y=y, y error=ey]{
						x y ey
						0  0.037231325754841615   0.02372742838770492
					};
					\addplot+ [olive, mark options={fill=olive}, mark=*] table [x expr={\thisrow{x} 0.2}, y=y, y error=ey] {
						x y ey
						0  0.12376492278886558 0.14837243451060964
					};
					\addplot+ [purple, mark options={fill=purple}, mark=*] table [x expr={\thisrow{x} 0.4}, y=y, y error=ey] {
						x y ey
						0 0.04078346136681322 0.02771149899706154
					};
					\coordinate (top) at (rel axis cs:0,1);
					\coordinate (bot) at (rel axis cs:1,0);
				\end{axis}
				\path (top-|current bounding box.west)-- 
				node[anchor=south,rotate=90] {Error} 
				(bot-|current bounding box.west);
				
			\end{tikzpicture}
		\end{subfigure}
	\end{minipage}
	\caption{Mean absolute error and standard deviation for estimating the total and direct causal changes of $X$ on $Y$. Subfigure (a) compares errors in total causal change estimation using (i) the set of parents of $X$, (ii) a minimal  set satisfying the back-door criterion, and (iii) the  set derived from Theorem~\ref{theorem:total_effect} and \ref{theorem:total_effect_cycle} using a difference graph. Subfigure (b) evaluates direct causal change estimation using (i) the set of parents of $Y$, (ii) a minimal  set satisfying the single-door criterion, and (iii) the  set from Theorem~\ref{theorem:direct_effect} and \ref{theorem:direct_effect_cycle}. Results are reported over $100$ pairs of dataset, consisting of $1000$ observations and $10$ variables, with each $10$ pairs sharing  the same difference graph for which the causal effect of interest is identifiable by our theorems. Each dataset is simulated using a linear SCM, where coefficients are randomly selected between $0$ and $1$, with Gaussian noise.}
	\label{fig:sim}
\end{figure}

\section{Related works and discussion}
\label{sec:rw}

There has been extensive research on causal reasoning using graphs that represent a class of graphs. For example, \cite{Maathuis_2013, Perkovic_2016, Perkovic_2020, Jaber_2022, Wang_2023} extended causal reasoning to partially oriented graphs such as CPDAGs and PAGs. PAGs are similar to CPDAGs but allow for hidden confounding. Both CPDAGs and PAGs represent multiple causal DAGs that share the same skeleton, some common edge orientations, and the same topological ordering. Moreover, extensive research has been conducted  for recovering these structures from observational data under specific assumptions~\citep{Spirtes_2000,Glymour_2019}.
From a different perspective, \cite{iwasaki_1994, chalupka_2016, Anand_2023, Assaad_2023, Ferreira_2024, Assaad_2024} have examined causal reasoning using cluster graphs, where each vertex corresponds to a cluster of low-level variables, and each edge denotes a causal relationship between two clusters. A cluster graph can represent multiple causal DAGs but does not necessarily share the same skeleton as the true causal DAG and is not necessarily acyclic. Like CPDAGs and PAGs, cluster graphs can also be discovered from data under certain assumptions~\citep{Wahl_2024}. 
In addition, other works, such as \cite{Richardson_1997, Forre_2020}, have specifically focused on cyclic directed graphs, where vertices do not represent clusters.
More closely related to difference graphs are selection diagrams~\citep{Bareinboim_2012} which can also represent differences between populations. There is a substantial body of literature on their use in causal reasoning~\citep{Bareinboim_2012} and methods for discovering them from data~\citep{Li_2023}. However, these two graphical representations serve different purposes. Selection diagrams provide more detailed information about stable causal relations and changes in noise, whereas difference graphs are more informative regarding changes in direct effects.
In summary, to date, there has been no investigation into causal reasoning using difference graphs. The existing work in the literature cannot be directly applied to difference graphs because the absence of an edge in these graphs has a different interpretation than in other types of graphs.

\textcolor{green}{This work addresses this gap and allow analysts to focus solely on the differences between two populations, simplifying reasoning, while still achieving identification for each domain individually.
	I emphasize that in this work, I assume that the difference graph is given and correct.
	The difference graph may be provided either by domain experts or by an algorithmic discovery method. In many domains, experts cannot construct a difference graph without knowing the true causal DAGs for both populations. However, in certain cases, this is feasible. For example, consider two cell populations: Population A (untreated) and Population B (treated with a drug that inhibits a specific protein). While the complete causal DAGs of these populations may be complex and unknown, domain knowledge about the drug’s direct mechanism of action allows researchers to construct a difference graph representing the expected changes in causal structure.
	When domain experts cannot directly specify a difference graph, they may rely on algorithmic approaches to discover it from data. Some algorithms discover an equivalence class of graphs~\citep{Wang_2018,Bystrova_2024_b}, while others aim to recover the exact difference graph~\citep{Chen_2023}. If the exact difference graph is known, all theorems intorduced in this paper, can be applied directly. For instance, consider the difference graph shown in \cite[Figure~4 (a)]{Chen_2023}, discovered from a real ovarian cancer dataset. Based on Theorem~\ref{theorem:total_effect}, the total effect of BIRC3 on PRKAR2B is not identifiable, whereas according to Theorem~\ref{theorem:direct_effect}, the direct effect of BIRC3 on PRKAR2B is identifiable, assuming linearity.
	In contrast, when working with Markov equivalence classes of difference graphs, the applicability of the results of this paper depends on whether the conditions of theorems hold across all graphs in the equivalence class.
	For instance, consider a scenario with four variables: $X$, $Y$, $W_1$, and $W_2$. If every difference graph within a given Markov equivalence class includes the edges $X \leftarrow W_1$, $X \rightarrow W_2$, and $X \rightarrow Y$, then the conditions of Theorem~\ref{theorem:total_effect} are satisfied across the entire equivalence class, ensuring that the total effect of $X$ on $Y$ is identifiable. Similarly, for the identification of the direct effect, suppose that in every difference graph within a given Markov equivalence class, the edges $W_1 \leftarrow Y$, $W_2 \rightarrow Y$, and $X \rightarrow Y$ are present. This guarantees that the conditions of Theorem~\ref{theorem:direct_effect} hold, allowing for the identification of $\alpha_{x,y}$.
}

\section{Conclusion}
\label{sec:conclusion}
This paper is the first to explore causal reasoning through the lens of difference graphs.
While difference graphs may not be as intuitive as classical causal DAGs for identifying causal effects, this work establishes conditions for identifying total effects (and total causal changes) in a nonparametric setting and direct effects (and direct causal changes)  in a linear setting, directly from the difference graph. 
%
For future works, it would be valuable to extend the results of this paper to direct effects in a nonparametric setting to path-specific effects.
Additionaly, if it will be shown that difference graphs with hidden confounding are useful and can be learned or provided by experts, then it will be valuable to investigate the identifiability conditions in the presence of hidden confounding.

\subsection*{Acknowledgements}
I thank Elias Bareinboim from Columbia University for the insightful discussions on difference graphs and selection diagrams.
This work was supported by the CIPHOD project (ANR-23-CPJ1-0212-01). 

\bibliographystyle{plainnat}
\bibliography{references.bib}




\end{document}